%% file: main2.tex
\documentclass[10pt, letterpaper]{article}

\usepackage[margin=1in]{geometry}
\usepackage{times} 
\usepackage[T1]{fontenc}
\usepackage[utf8]{inputenc}
\usepackage{authblk} 
\usepackage{abstract}

\usepackage{amsmath,amssymb,amsfonts,amsthm}
\usepackage{graphicx}
\usepackage{booktabs}
\usepackage{multirow}
\usepackage{tabularx}
\usepackage[colorlinks=true, allcolors=blue]{hyperref}
\usepackage[nocompress]{cite}

\usepackage{titlesec}
\titleformat{\section}{\large\bfseries}{\thesection.}{1em}{}
\titleformat{\subsection}{\normalsize\bfseries}{\thesubsection}{1em}{}

\title{\textbf{Beyond Memorization: Selective Learning for Copyright‑Safe Diffusion Model Training
}}

\author{\textbf{Divya Kothandaraman}}
\author{\textbf{Jaclyn Pytlarz}}
\affil{Dolby Laboratories}
\date{}
\newtheorem{corollary}{Corollary}[section]
\newtheorem{theorem}{Theorem}[section]

\begin{document}

\maketitle

\input{Sections/0-abstract}
\input{Sections/1-introduction}
\input{Sections/2-related}

\input{Sections/3-method}
\input{Sections/4-experiments}

\input{Sections/5-conclusion}

\bibliographystyle{plain}
\bibliography{references}

\end{document}

%% file: Sections/0-abstract.tex
\begin{abstract}

    Memorization in large-scale text-to-image diffusion models poses significant security and intellectual property risks, enabling adversarial attribute extraction and the unauthorized reproduction of sensitive or proprietary features. While conventional dememorization techniques, such as regularization and data filtering, limit overfitting to specific training examples, they fail to systematically prevent the internalization of prohibited concept-level features. Simply discarding all images containing a sensitive feature wastes invaluable training data, necessitating a method for selective learning at the concept level. 
    
    We introduce a gradient projection method designed to enforce a stringent requirement of concept-level feature exclusion. Our defense operates during backpropagation by systematically identifying and excising training signals aligned with embeddings of prohibited attributes. Specifically, we project each gradient update onto the orthogonal complement of the sensitive feature's embedding space, thereby zeroing out its influence on the model's weights. Our method integrates seamlessly into standard diffusion model training pipelines and complements existing defenses. We analyze our method against an adversary aiming for feature extraction. In extensive experiments, we demonstrate that our framework drastically reduces memorization while rigorously preserving generation quality and semantic fidelity. By reframing memorization control as selective learning, our approach establishes a new paradigm for IP-safe and privacy-preserving generative AI.
\end{abstract}

%% file: Sections/1-introduction.tex
\section{Introduction}

Generative vision models inevitably learn~\cite{somepalli2023diffusion,carlini2023extracting} both valuable patterns and unwanted, often proprietary, attributes from their training data. When specific visual concepts, be it a person’s likeness, a trademarked logo, or a distinctive artistic style, must never be reproduced, simply removing every image that contains them is a blunt security instrument. This approach discards rich training data, significantly weakening the model’s capabilities and raising the cost of dataset curation. More critically, even if exact memorization of a training image is avoided, an adaptive adversary~\cite{ren2025reverse} can still manipulate seeds, noise inputs, or carefully crafted prompts to extract sensitive content or close approximations, revealing the model's internalized secrets.

This vulnerability to adversarial feature extraction demonstrates that conventional training-time defenses~\cite{ren2024unveiling} like regularization and data filtering, which aim to limit general overfitting, often leave residual feature traces. Similarly, inference-side techniques~\cite{gandikota2023erasing} like adversarial fine-tuning and unlearning, which operate as post-hoc methods, have been repeatedly shown in empirical studies to be vulnerable to concept recurrence, making it exceedingly difficult to reliably purge a learned concept from a model’s internal feature space.

What is fundamentally needed is selective learning: a robust, security-focused mechanism for models to ingest all of the beneficial content in an image while systematically blocking particular concept-level signals from ever taking root. Most prior work on memorization focuses on stopping the model from overfitting to exact training examples. Our formulation goes a critical step further: we propose a defense that blocks the model from encoding any identifiable attributes of protected content, the core features of a concept, not just the specific instance in the training set. We offer a novel perspective by reframing the problem as one of blocking concept acquisition rather than merely pruning memorized instances. This shift transforms memorization control into a provable, concept-level security constraint. By rigorously defining and enforcing this selective learning mechanism, we offer a contribution to the field of responsible and IP-safe generative modeling.

Before presenting our defense, we must first establish the security landscape: why does reducing general memorization not guarantee that models will abstain from reproducing restricted concepts? We begin by framing memorization in generative AI, identifying the training dynamics that drive models to internalize unwanted examples, and briefly surveying established techniques for curbing memorization~\cite{chen2024towards,somepalli2023understanding}, underscoring why they are insufficient against concept-level extraction attacks.

We then pivot to the core technical challenge: how can we enforce selective learning in text-to-image diffusion models? To tackle this, we present a novel method based on gradient projection. We demonstrate how projecting gradients away from directions tied to restricted concepts acts as an intrinsic safeguard, ensuring that the model never internalizes those undesirable attributes during the critical weight update step. This procedure integrates seamlessly into the existing diffusion training pipeline. Crucially, selective learning via gradient projection is not a substitute for standard memorization mitigation; it complements them by providing an explicit, targeted defense mechanism to block the acquisition and reproduction of restricted concepts. This layered approach ensures that while general overfitting is addressed with conventional methods, the model is also prevented from learning forbidden features through our precise gradient intervention.

Through extensive experiments, we demonstrate that selective learning sharply reduces the replication of targeted attributes while preserving generation fidelity and diversity. We validate the robustness of our method across multiple prohibited-concept scenarios and confirm its effectiveness against adversarial extraction attacks.

\subsection{Main Contributions}

\begin{itemize}
    \item We introduce the concept of Selective Learning, formally framing model training as a process that ingests beneficial patterns while blocking the acquisition of specified concept-level signals, thereby providing a stronger security guarantee than conventional unlearning.
    \item We propose a novel method based on Gradient Projection for text-to-image diffusion models that steers learning dynamics away from restricted-concept directions, ensuring those attributes are never internalized. By treating forbidden concepts as first-class constraints rather than post-hoc filters, we introduce a selective learning paradigm that ensures that text-to-image models generate only desired content.
    \item We show how to seamlessly integrate selective learning into standard diffusion-model training pipelines, complementing existing regularization and data-filtering defenses with a targeted, intrinsic mechanism.
    \item Through extensive experiments, we demonstrate that our approach sharply reduces replication of prohibited attributes without degrading generation fidelity or diversity. We validate the robustness of our method under adversarial extraction attacks, confirming the effectiveness of our gradient-projection–based selective learning method.
\end{itemize}

%% file: Sections/2-related.tex
\section{Rethinking Memorization Limits for Selective Learning}

We posit that reducing generic memorization~\cite{ippolito2023preventing} by itself does not satisfy the requirements of concept-level selective learning. To understand this gap, we first analyze the mechanism and limitations of conventional memorization defenses, establishing why they fail to provide robust security guarantees against targeted concept extraction.

\subsection{Background: Memorization Dynamics and Attack Surface}

Memorization~\cite{arpit2017closer,schwarzschild2024rethinking,somepalli2023diffusion} refers to a model's tendency to reproduce training examples or their close variants, creating a critical attack surface where an adversary can exploit the model's overfitting to extract sensitive, proprietary, or private training data.

\subsubsection{Definition of Memorization in Generative Models}

Memorization arises when a model’s capacity and training regimen exceed what’s needed for generalization, causing high-capacity networks to act as instance-specific lookup tables. From a learning theory perspective, memorization corresponds to low generalization performance despite a low empirical risk. The empirical risk minimization objective function
\begin{equation}
  \hat R(f) \;=\; \frac{1}{N}\sum_{i=1}^N \ell\bigl(f(x_i), y_i\bigr)
\end{equation}
can be driven arbitrarily close to zero if $f$ has enough flexibility. Overly expressive models exhibit low bias but extreme variance, fitting idiosyncratic noise.

Diffusion models learn to reverse a gradual noising process by estimating the score function $\nabla_x \log p_t(x)$. The most common training loss is the denoising score matching objective:
\begin{equation}
  \mathcal{L}(\theta)
  = \mathbb{E}_{t,\,x_0,\,\epsilon}
    \Bigl[\bigl\|\epsilon - \epsilon_\theta(x_t, t)\bigr\|^2\Bigr],
\end{equation}
where the noisy sample $x_t$ is defined by
\begin{equation}
  x_t = \sqrt{\bar\alpha_t}\,x_0 \;+\; \sqrt{1 - \bar\alpha_t}\,\epsilon.
\end{equation}
\noindent If $\epsilon_\theta$ overfits, it learns a near instance-specific mapping from $x_t$ to the exact noise $\epsilon$. This effectively embeds training examples in the learned score field. During sampling, the iterative reverse‑diffusion process repeatedly queries this biased score estimate, causing trajectories to drift toward these memorized regions of the data manifold. In the absence of explicit privacy constraints, the model is therefore incentivized to ``remember'' training examples rather than learn a generalizable score function.

\subsection{Mitigating Memorization: Prior Work and Security Drawbacks}

Existing mitigation techniques, classified as training-time and inference-time strategies, fail to provide the necessary targeted, concept-level exclusion guarantee against an adaptive adversary.

\subsubsection{Training-Time Defenses: Failure Modes}

\paragraph{Data-Level Curation}
Techniques like exact and near-duplicate removal, controlled caption granularity, and data augmentation reduce the risk of rote reproduction \cite{chen2024towards,somepalli2023understanding}. However, adversaries~\cite{ma2024could,balle2022reconstructing,carlini2023extracting,ma2025jailbreaking, jegorova2022survey} can sidestep these defenses by crafting prompts that exploit subtle invariances or by combining noise schedules to find copyrighted patches.

\paragraph{Model-Level Interventions}
Strategies such as early stopping, per-sample loss filtering, and Differentially Private SGD (DP-SGD) each impede over-adaptation \cite{ren2024unveiling,chen2024exploring,chen2025enhancing,bonnaire2025diffusion,dockhorn2022differentially,zhu2018hidden}. In practice, early stopping~\cite{garnier2025early,favero2025bigger} only postpones memorization. Norm‐ and loss‐threshold detectors can be gamed by smoothing gradient signatures. Although DP‑SGD~\cite{nasr2021adversary} provides formal privacy guarantees, achieving a meaningful privacy budget often requires injecting substantial noise during training, which can significantly degrade model utility. Moreover, if the privacy parameters are set too loosely, the resulting model may still leak information that attackers can exploit, potentially by aggregating many model outputs, even though the DP guarantee itself does not weaken under such aggregation. Ambient diffusion \cite{daras2023ambient} reduces direct memorization but its lack of targeted control means identifiable features of copyrighted entities can persist in the latent space and later be recovered. Techniques like parameter-efficient fine-tuning and low-rank adapters (LoRA) \cite{hu2022lora,li2023loftq,han2024parameter,hayou2024lora+} and ensemble defenses \cite{liu2024iterative} cap the memorization budget. Despite these measures, optimization routines can still navigate memorization subspaces embedded in the weight manifold.

\subsubsection{Inference-Time Strategies}

Inference‑time, post‑hoc guardrails, such as prompt filtering, output pruning, heuristic perturbations, or watermarking, operate only at the interface level and leave the model weights untouched. Because the underlying generative distribution remains unchanged, adversarially crafted prompts can still navigate the model’s latent space and extract memorized or near‑verbatim copyrighted content \cite{ma2024inversion,kowalczuk2025finding,ma2024could,tsai2023ring}. Pruning~\cite{ni2025controllable} fails because memorized concepts are distributed across the network. Post-training fine-tuning and unlearning methods suffer from concept recurrence: once a concept is woven into the high-dimensional feature manifold, it is extremely difficult~\cite{zhang2024generate,suriyakumar2024unstable,george2025illusion} to surgically remove without collateral damage to generalization. Adversaries~\cite{george2025illusion} can exploit residual correlations to resurrect ``forgotten'' content.

\subsection{Mitigating Memorization Alone Cannot Guarantee Exclusion of Restricted Concepts}

\begin{theorem}[Total‐Variation Leakage]
Let $p_\theta$ be a learned distribution and $p_{\mathrm{data}}$ the training distribution. If
\[
\mathrm{TV}(p_\theta, p_{\mathrm{data}})
= \sup_{A\subset\mathcal X}
\bigl|\Pr_{x\sim p_\theta}[x\in A]
- \Pr_{x\sim p_{\mathrm{data}}}[x\in A]\bigr|
\le \delta,
\]
then for every measurable set $S\subset\mathcal X$,
\[
\Pr_{x\sim p_\theta}[x\in S]
\;\ge\;
\Pr_{x\sim p_{\mathrm{data}}}[x\in S] \;-\;\delta.
\]
\end{theorem}

\begin{proof}
By definition of total variation,
\[
\bigl|\Pr_{p_\theta}[S] - \Pr_{p_{\mathrm{data}}}[S]\bigr|
\;\le\;\mathrm{TV}(p_\theta, p_{\mathrm{data}})\le \delta,
\]
so $\Pr_{p_\theta}[S]\ge\Pr_{p_{\mathrm{data}}}[S]-\delta$.
\end{proof}

\begin{corollary}[Repeated‐Sampling Amplification]
If $\Pr_{x\sim p_{\mathrm{data}}}[x\in S]=\alpha>0$, then drawing $N$ i.i.d. samples $x_1,\dots,x_N\sim p_\theta$ yields
\[
\Pr\bigl[\exists\,i:\,x_i\in S\bigr]
=1 - \bigl(1 - \Pr_{x\sim p_\theta}[x\in S]\bigr)^N
\;\ge\;1 - (1 - \alpha + \delta)^N,
\]
which tends to 1 as $N\to\infty$ whenever $\alpha>\delta$.
\end{corollary}

\begin{proof}
Independence gives $\Pr[\forall i:\,x_i\notin S]
= \bigl(1 - \Pr_{p_\theta}[S]\bigr)^N$,
then substitute $\Pr_{p_\theta}[S]\ge\alpha-\delta$.
\end{proof}

\subsubsection*{Applications}

\paragraph{Copyrighted‐Content Extraction}
Even if the learned distribution of a diffusion model $p_\theta$ is arbitrarily close to the training distribution $p_{\mathrm{data}}$ in a statistical sense, an adversary can still extract copyrighted images. We make this precise with total‐variation bounds and repeated‐sampling arguments. Let $S$ be the set of copyrighted images with prevalence $\alpha=\Pr_{x\sim p_{\mathrm{data}}}[x\in S]$.
Even if $\delta$ is vanishingly small, repeated sampling from $p_\theta$ recovers copyrighted content with probability $\ge1-(1-\alpha+\delta)^N\to1$ whenever $\alpha>\delta$.

\paragraph{Semantic Identity Leakage}
Even when a diffusion model is regularized so that it does not reproduce any training image verbatim, it can still capture and regenerate the ``identity'' of a subject (e.g.\ Tom Cruise). We make this precise by showing that any small statistical distance to the true data distribution entails a nontrivial probability of generating samples recognized as that identity. Define an oracle classifier $C:\mathcal X\to\{0,1\}$ for a recognizable subject (e.g.\ Tom Cruise), and let
$\alpha=\Pr_{x\sim p_{\mathrm{data}}}[C(x)=1]$.
Setting $S=\{x:C(x)=1\}$ in the theorem yields
$\Pr_{x\sim p_\theta}[C(x)=1]\ge\alpha-\delta$,
and by the corollary, sampling $N$ times detects the identity with probability $\ge1-(1-\alpha+\delta)^N\to1$ as soon as $\alpha>\delta$.

\noindent\textbf{Discussion.} Preventing exact memorization does not block \emph{semantic} leakage: the model can still learn and reproduce the subject’s identity in novel contexts because the prevalence of the initial concept is non-zero and cannot be eliminated without destructive underfitting. Our method addresses this by enforcing concept acquisition blocking.

%% file: Sections/3-method.tex
\section{Beyond Memorization: Gradient Projection for Selective Learning}

\begin{figure*}[h]
    \centering
    \includegraphics[width=1.0\linewidth]{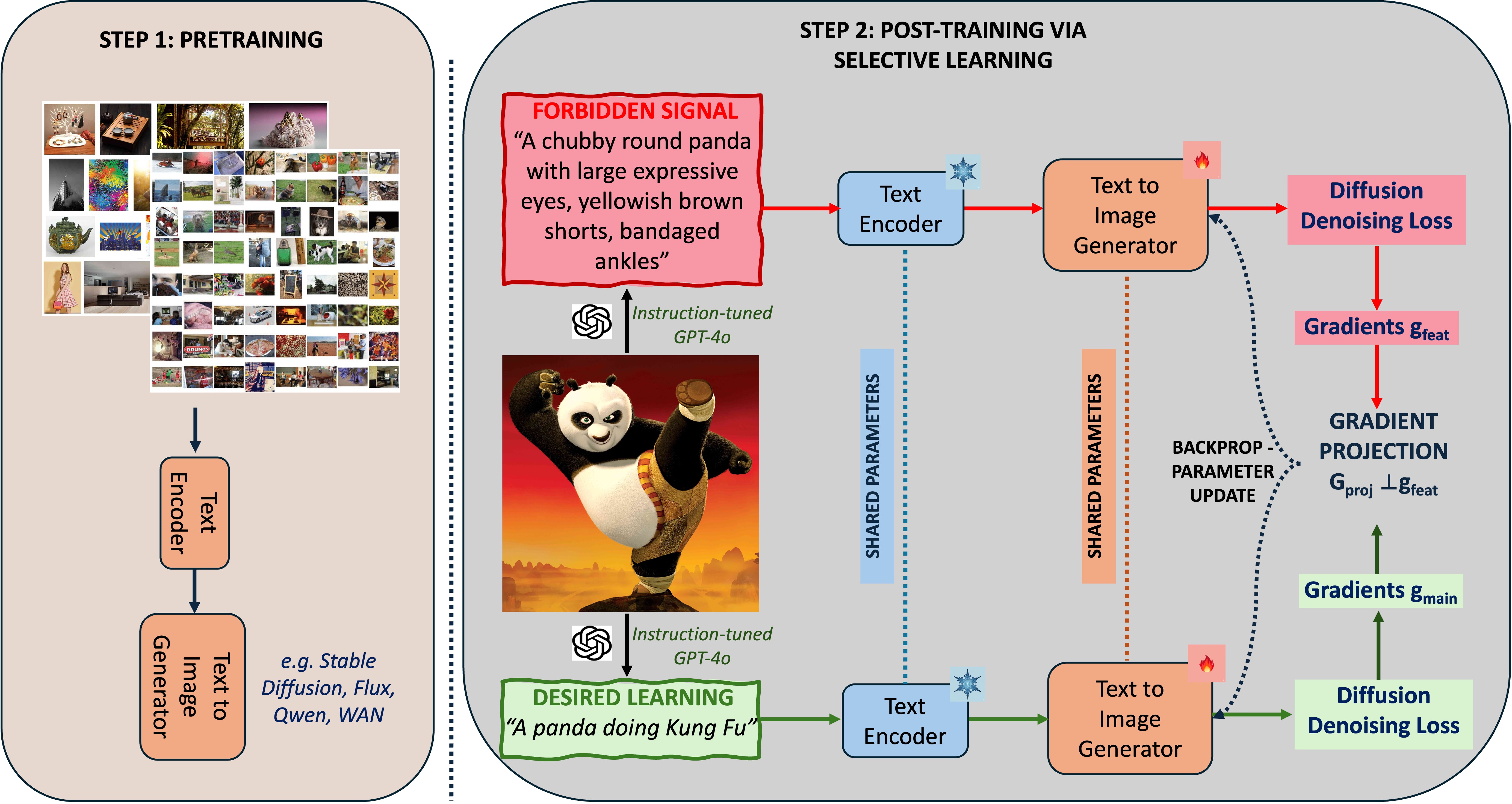}
    \caption{\textbf{Overview of the Gradient Projection Framework for Selective Learning.} This two-stage protocol employs (Step 1) Pretraining on general data, followed by (Step 2) Constrained Fine-Tuning on sensitive examples. At each update, the framework calculates the gradient for the Forbidden Signal ($g_{feat}$) and then orthogonally projects the Desired Learning gradient ($g_{main}$) onto its complement. This geometric constraint, $G_{proj} \perp g_{feat}$, ensures the model learns beneficial content while provably blocking the acquisition of restricted concept-level features, thereby enforcing intellectual property (IP) preservation during training.}
    \label{fig:overview}
\end{figure*}

The key mechanism in selective learning is to modify gradient updates so that training steers the model away from unwanted concepts. Unlearning methods~\cite{wu2025unlearning,ren2024unveiling} perform this adjustment in a post-hoc phase to erase specific knowledge, while memorization defenses~\cite{somepalli2023understanding, favero2025bigger} operate during training to prevent overfitting to individual samples. However, these approaches, including, noise injections, extra tokens, generic regularizers, are at best heuristic. They might obscure forbidden patterns, but they offer no formal guarantee that protected features won’t leak back in under adversarial prompts or clever inference hacks. Hence, there is a pressing need for a method that explicitly prevents diffusion models from learning sensitive motifs during training.

To address the challenge of enforcing IP-safe generative modeling, we propose a gradient projection-based method. Gradient projection~\cite{calamai1987projected} is a classical technique used in iterative optimization to enforce constraints by projecting the gradient onto a feasible subspace, effectively removing components that violate those constraints. This approach has been successfully applied across various domains~\cite{tian2023trainable,saha2021gradient,qiao2025gradient,hoang2024learn,huang2024unified,wang2021adversarial}, including fairness in machine learning, adversarial robustness by steering updates away from vulnerable directions, and privacy-focused applications such as machine unlearning.

Building on this foundation, we argue that gradient projection is well-suited for mitigating unauthorized feature acquisition in generative models. Its geometric enforcement proves resilient against adversarial manipulations, including prompt variation, noise schedule perturbations, and inference-time exploits, as excluded directions cannot re-enter the optimization process. Consequently, the model’s capacity for rote memorization becomes provably non-increasing, with each weight update explicitly excising forbidden directions.

Our method operates as follows. After computing the raw gradient update, we identify a forbidden feature subspace corresponding to copyrighted or otherwise sensitive motifs. The gradient is then orthogonally projected onto the complement of this subspace, effectively removing any aligned component. This approach adapts established projected-gradient techniques~\cite{calamai1987projected}, ensuring consistent reduction of memorization potential making our method a robust addition to any diffusion-model training pipeline. 

\subsection{Method: Enforcing Concept Exclusion via Orthogonal Projection}

To apply gradient projection in a text-conditioned diffusion model, we first isolate the visual features that pose an IP-leakage risk. Consider an image of Kung Fu Panda’s Po. Our objective is to let the model learn the general concept of ``a panda doing Kung Fu'' while preventing it from acquiring Po’s distinctive, protected likeness (e.g., rounded ear shape, yellowish-brown shorts, characteristic facial features).

Because the model is text-conditioned, we use the CLIP text encoder to define directions in concept space. We begin by selecting a primary caption $p_{\mathrm{main}}$ that captures the desired concept:

\[
p_{\mathrm{main}} = \text{``a panda doing Kung Fu.''}
\]

Using the denoising loss $\mathcal{L}$, we compute the corresponding training gradient:

\[
g_{\mathrm{main}}
=
\nabla_{\theta}\,\mathcal{L}\bigl(x,\;p_{\mathrm{main}};\theta\bigr),
\]

which, if left unmodified, would encourage the model to learn all features present in the training image, including Po’s copyrighted appearance.

To identify the directions responsible for memorizing protected attributes, we introduce an auxiliary prompt $p_{\mathrm{feat}}$ that isolates a specific copyrighted feature:

\[
p_{\mathrm{feat}} = \text{``a panda wearing yellowish-brown shorts.''}
\]

The gradient associated with this prompt,

\[
g_{\mathrm{feat}}
=
\nabla_{\theta}\,\mathcal{L}\bigl(x,\;p_{\mathrm{feat}};\theta\bigr),
\]

serves as the IP-leakage direction we aim to remove. We then project $g_{\mathrm{main}}$ onto the orthogonal complement of $g_{\mathrm{feat}}$. With $\varepsilon$ added for numerical stability, the orthogonalized gradient is

\[
g_{\perp}
=
g_{\mathrm{main}}
- \lambda\,
\frac{\langle g_{\mathrm{main}},\,g_{\mathrm{feat}}\rangle}
{\|g_{\mathrm{feat}}\|^2 + \varepsilon}
\,g_{\mathrm{feat}}.
\]

The hyperparameter $\lambda$ controls the strength of the projection. 
Under $\lambda=1$, the resulting $g_{\perp}$ contains no component aligned with Po’s protected attributes. To maintain consistent learning dynamics, we rescale the orthogonalized gradient to match the norm of the original update:

\[
g_{\mathrm{proj}}
=
\frac{\|g_{\mathrm{main}}\|}{\|g_{\perp}\|}\;g_{\perp}.
\]

Finally, we update the model parameters using this constrained gradient:

\[
\theta \leftarrow \theta - \eta\,g_{\mathrm{proj}},
\]

where $\eta$ is the learning rate. This ensures that each step reinforces the intended concept while provably preventing the model from increasing its capacity to reproduce copyrighted features.

\begin{figure*}
    \centering
    \includegraphics[width=0.99\linewidth]{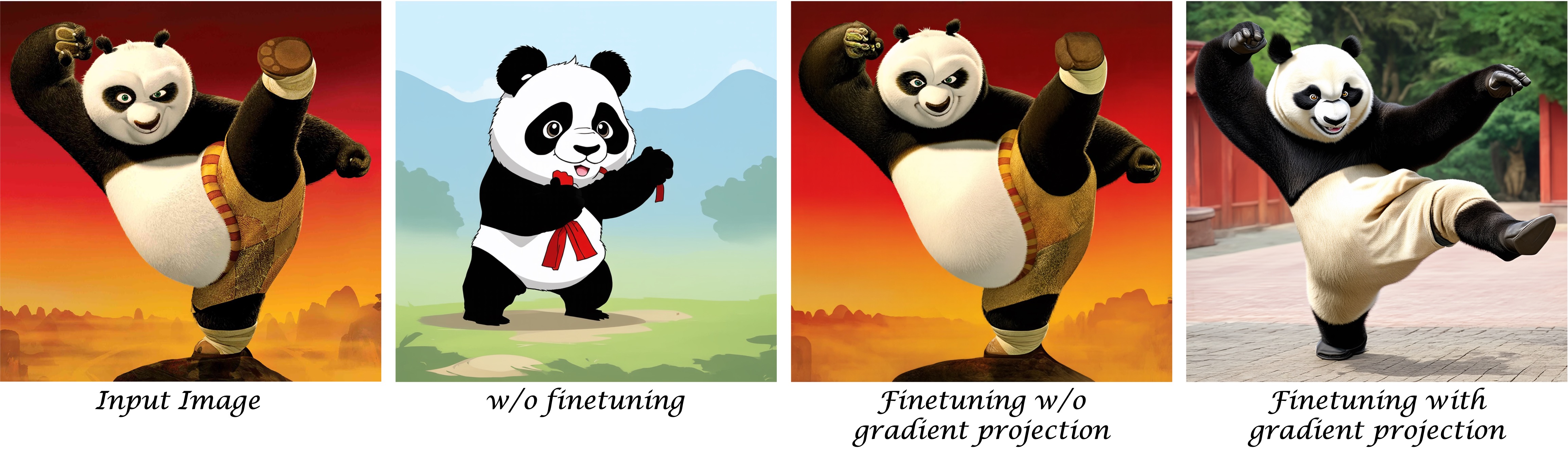}
    \caption{Generation results on the “Kung Fu Panda” single-image case study using the main prompt “a panda doing kung fu” and the forbidden prompt “A huge panda wearing yellow shorts with a red and yellow waistband, with small ears, bandaged ankles, distinct fur patterns.” (a) Before finetuning, the base model produces a generic kung fu panda. (b) After naïve LoRA finetuning, the model memorizes and faithfully reproduces the copyrighted visual details from the reference image. (c) With gradient-projection finetuning, the model captures only the abstract concept of a kung fu panda without emitting any of the original image’s sensitive attributes.}
    \label{fig:result1}
\end{figure*}
\subsection{Analysis: Formal Security Guarantees}

The following theorems formalize this guarantee: when $\lambda = 1$, the projected gradient is exactly orthogonal to the forbidden feature subspace, removing all first‑order learning signals along those directions in each constrained update.

\subsubsection{Zero First‐Order Improvement on Forbidden Features}

\begin{theorem}[Gradient Projection Prevents First-Order Learning]
\label{thm:gradient_projection}
Let $g_{\text{feat}} = \nabla_\theta \mathcal{L}(x, p_{\text{feat}}; \theta)$ be the gradient for a forbidden feature, and let $g_{\text{proj}}$ be the projected gradient orthogonal to $g_{\text{feat}}$. Then:
$$\langle g_{\text{proj}}, g_{\text{feat}} \rangle = 0$$
This ensures zero first-order improvement in the forbidden feature direction.
\end{theorem}

\begin{proof}
By construction of the gradient projection (ignoring $\varepsilon$ for exact proof):
$$g_{\text{proj}} = g_{\text{main}} - \frac{\langle g_{\text{main}}, g_{\text{feat}} \rangle}{\|g_{\text{feat}}\|^2} g_{\text{feat}}$$
Computing the inner product with $g_{\text{feat}}$:
\begin{align*}
\langle g_{\text{proj}}, g_{\text{feat}} \rangle &= \left\langle g_{\text{main}} - \frac{\langle g_{\text{main}}, g_{\text{feat}} \rangle}{\|g_{\text{feat}}\|^2} g_{\text{feat}}, g_{\text{feat}} \right\rangle \\
&= \langle g_{\text{main}}, g_{\text{feat}} \rangle - \frac{\langle g_{\text{main}}, g_{\text{feat}} \rangle}{\|g_{\text{feat}}\|^2} \langle g_{\text{feat}}, g_{\text{feat}} \rangle \\
&= \langle g_{\text{main}}, g_{\text{feat}} \rangle - \frac{\langle g_{\text{main}}, g_{\text{feat}} \rangle}{\|g_{\text{feat}}\|^2} \|g_{\text{feat}}\|^2 \\
&= \langle g_{\text{main}}, g_{\text{feat}} \rangle - \langle g_{\text{main}}, g_{\text{feat}} \rangle = 0
\end{align*}
Therefore, the directional derivative of the forbidden feature loss in the direction of the update $\Delta\theta = -\eta g_{\text{proj}}$ is:
$$\frac{d}{d\eta} \mathcal{L}_{\text{feat}}(\theta - \eta g_{\text{proj}})\Big|_{\eta=0} = -\langle \nabla_\theta \mathcal{L}_{\text{feat}}, g_{\text{proj}} \rangle = -\langle g_{\text{feat}}, g_{\text{proj}} \rangle = 0$$
\end{proof}

The first-order constraint in Theorem 3.1 is applied iteratively across all $T$ updates. By enforcing $\langle g_{proj}, g_{feat} \rangle = 0$ at each step, we ensure the parameter state $\theta_T$ remains within the orthogonal complement of the forbidden subspace $S_f$ relative to the initialization $\theta_0$.

\subsubsection{Invariance of the Forbidden‐Subspace Component}

Let $S = \mathrm{span}\{g_{\mathrm{feat}}\}$ denote the memorization subspace, and define the projector onto $S$ as:
\[
P_S \;=\; \frac{g_{\mathrm{feat}}\,g_{\mathrm{feat}}^{T}}{\|g_{\mathrm{feat}}\|^2}.
\]
Since the update direction $\Delta\theta = -\eta g_{\mathrm{proj}}$ lies in the orthogonal complement of $S$, we have $P_S\,\Delta\theta = 0$. Consequently:
\[
P_S\,\theta_{t+1}
=
P_S\bigl(\theta_t + \Delta\theta\bigr)
=
P_S\,\theta_t.
\]
Thus, the component of the model’s parameters along the forbidden direction remains \emph{constant} across every update.

\subsubsection{Bounded Memorization Capacity: Geometric Interpretation}

In diffusion training, the risk of reproducing a specific feature can be approximated by the projection of the parameter vector onto the feature's gradient direction. Under a first-order (NTK) approximation, we define the memorization capacity $M_f(\theta)$ of a forbidden feature $f$ as

\[
M_f(\theta) \;=\; \|P_{S_f}\theta\|^2,
\qquad
S_f = \mathrm{span}\{g_{\mathrm{feat}}\},
\]

where $P_{S_f}$ denotes the orthogonal projector onto the forbidden feature subspace.

Geometrically, $M_f(\theta)$ measures the alignment of the parameter state with the feature's Fisher-sensitive direction. In the Neural Tangent Kernel (NTK) view~\cite{jacot2018neural,tishby2015deep}, this alignment quantifies the model's directional sensitivity to the protected attribute.

By constraining each update $\Delta\theta$ to lie in $S_f^{\perp}$ (i.e., $P_{S_f}\,\Delta\theta = 0$), we ensure that the projection onto $S_f$ does not increase:

\[
P_{S_f}(\theta_{t+1}) \;=\; P_{S_f}(\theta_t).
\]

Thus, the model’s first-order capacity to encode or reconstruct the forbidden feature remains bounded by its initial value. This provides a structural safeguard against feature memorization that is independent of standard overfitting metrics.

\begin{theorem}[Memorization Capacity Bounds]
\label{thm:capacity_bounds}
Let $f$ be a forbidden feature represented by a subspace $S_f = \text{span}\{g_{\text{feat}}\}$, and let $\Pi_f$ denote the orthogonal projection onto $S_f$. Define the memorization capacity of feature $f$ at parameter state $\theta$ as:
$$M_f(\theta) = \|\Pi_f \theta\|^2$$
Then, under gradient updates constrained to be orthogonal to $S_f$, we have:
$$M_f(\theta_{t+1}) \leq M_f(\theta_t)$$
with equality when the projection is computed exactly.
\end{theorem}

\begin{proof}
Assume the model parameters are updated via $\theta_{t+1} = \theta_t - \eta g_{\text{proj}}$, where $g_{\text{proj}} \perp S_f$.
Since $\Pi_f$ projects onto $S_f$, and $g_{\text{proj}}$ is orthogonal to $S_f$:
$$\Pi_f g_{\text{proj}} = 0$$
Applying the projection to the updated parameters:
\begin{align*}
\Pi_f \theta_{t+1} &= \Pi_f(\theta_t - \eta g_{\text{proj}}) \\
&= \Pi_f \theta_t - \eta \Pi_f g_{\text{proj}} \\
&= \Pi_f \theta_t
\end{align*}
Therefore, the memorization capacity remains unchanged:
$$M_f(\theta_{t+1}) = \|\Pi_f \theta_{t+1}\|^2 = \|\Pi_f \theta_t\|^2 = M_f(\theta_t)$$
In practice, due to numerical imprecision and $\varepsilon$ stabilization, we observe $M_f(\theta_{t+1}) \leq M_f(\theta_t) + O(\epsilon_{\text{num}})$.
\end{proof}

\paragraph{Interpretation.} This result shows that by controlling the direction of gradient updates, we prevent the model from increasing its alignment with sensitive features, ensuring the memorization capacity is frozen at its initial value, hence bounding the reproduction risk.

\subsubsection{Robustness Against Adversarial Prompts}

Any adversarial prompt that attempts to elicit the forbidden feature will generate a corresponding gradient $g'_{\mathrm{feat}}\in S$. The same projection step strips out its component, ensuring no new forbidden‐feature information can ever reenter training.

\paragraph{Conclusion.} By enforcing $g_{\mathrm{proj}}\;\perp\;g_{\mathrm{feat}}$, we achieve (i) zero first‐order improvement on forbidden features, (ii) invariance of forbidden‐subspace components, (iii) a hard cap on memorization capacity, and (iv) resilience to adversarial prompts. This delivers a formal, provable safeguard against IP leakage that heuristic methods cannot match.

\subsection{Training Protocol: A Divide-and-Conquer Strategy}

We recommend a two‐stage regimen for training text-based vision generation models with robust copyright protection.

\begin{enumerate}
    \item \textbf{Stage 1: General Pretraining.} Partition the dataset into non‐sensitive and copyright‐sensitive subsets. Pretrain the model in the usual way on the non‐sensitive data. No gradient projection is needed here, as the goal is to learn general visual and semantic patterns.
    \item \textbf{Stage 2: Constrained Fine-Tuning.} Once pretraining is complete, conduct a fine-tuning pass on the copyright-sensitive images (using either the full model or LoRA layers). Crucially, apply gradient projection exclusively on the copyright-sensitive images to suppress any update components aligned with protected features.
\end{enumerate}

This divide‐and‐conquer strategy is essential because gradient projection relies on meaningful feature gradients for the attributes you wish to suppress. An untrained model cannot yet recognize sensitive attributes, so computing a feature gradient at random initialization would be uninformative. After general pretraining, the model has internalized broad visual concepts and can generate targeted gradients for those sensitive attributes, making the subsequent projection both effective and precise.

As a word of caution, this copyright‐preserving protocol should \textbf{complement, not replace} standard memorization safeguards such as data de-duplication, caption specificity controls, regularization and other techniques. Those measures remain crucial for preventing generic overfitting and memorization in diffusion training. The gradient‐projection step is intended as an additional module that explicitly enforces non-memorization of the sensitive content.


%% file: Sections/4-experiments.tex
\section{Experimental Analysis}

\input{Tables/results1}

\subsection{Experimental Settings and Data Curation}

We curate a specialized dataset from the OpenVid~\cite{nan2024openvid} video corpus to create training samples prone to high-risk IP leakage. Each example pairs a frame with two complementary captions: one to reinforce desired, abstract content and one to suppress forbidden, copyright-sensitive features.

\paragraph{Data Preparation}
We apply the following procedure to select images with a high risk of copyright infringement or privacy leakage:
\begin{itemize}
    \item \textbf{Keyword Filter (IP/Privacy Risk):} We select videos whose captions include \texttt{"animated"} (indicating stylized content subject to copyright) or \texttt{"person"} (flagging potential privacy and rights concerns around faces).
    \item \textbf{Frame Extraction:} From each filtered video, we extract the exact middle frame to serve as the canonical training sample.
\end{itemize}

Following prior work \cite{webster2023reproducible}, we select a total of $2413$ images identified as highly memorized and copyrighted for our evaluation. This dataset size aligns with targeted memorization studies~\cite{ren2024unveiling,kowalczuk2025finding}, enabling computationally tractable evaluation.

\paragraph{Prompt Generation: Defining the Forbidden Subspace}
A potential concern in selective learning is the manual overhead required to define forbidden prompts ($p_{feat}$) for diverse datasets. Our framework addresses this through an automated, LLM-driven pipeline that disentangles protected attributes from abstract scene elements at scale. Given a predefined set of IP-safety rules or a high-level description of protected entities, an instruction-tuned LLM (e.g., GPT-OSS 20B) systematically analyzes the original metadata to generate pairs of $p_{main}$ and $p_{feat}$. This procedure eliminates the need for human-in-the-loop annotation, allowing the Gradient Projection framework to be integrated into massive-scale diffusion training pipelines without a linear increase in manual labor. We use the following instruction with the original video caption (\verb|{original_caption}|):

\begin{verbatim}
Analyze the following caption and identify 
entities that would be copyright-sensitive 
in the corresponding image:
{original_caption}
1. Write a new prompt without these 
copyright-sensitive entities (<= 77 tokens).
2. Write a prompt using only the copyright-
sensitive entities (<= 77 tokens).
\end{verbatim}

\begin{itemize}
    \item Output (1) becomes the \textbf{Main Caption ($p_{\mathrm{main}}$)} guiding positive concept learning.
    \item Output (2) becomes the \textbf{Forbidden Caption ($p_{\mathrm{feat}}$)} specifying the sensitive entities to suppress via gradient projection.
\end{itemize}

\paragraph{Example Pairing}
The example pairing (Main Caption: ``A vibrant cartoon-style interior, wooden floor, tall bookshelf, wooden table with teapot, striped chair, large golden vase on ornate walls, light streaming with whimsical décor, colorful palette, dynamic lighting.''; Forbidden Caption: ``purple animated character wearing white hat, blue scarf, holding sword, looking down, in a room with wooden floor, bookshelf full of books, wooden table with teapot and bowl, striped pattern chair, framed pictures on walls, large golden vase, cartoonish colorful style.'') demonstrates how we separate abstract scene elements from specific, high-risk character attributes.

\paragraph{Evaluation Metrics}
Following prior work~\cite{somepalli2023understanding, hintersdorf2024finding}, we use three complementary metrics to assess our method while maintaining generation quality:
\begin{itemize}
    \item \textbf{SSCD (Self-Supervised Copy Detection)~\cite{pizzi2022self}:} Our primary metric for copyright preservation, SSCD employs a specialized neural network to detect copied or near-duplicate content across visual transformations. Lower SSCD scores indicate successful prevention of unauthorized feature reproduction and stronger defense against potential IP infringement.
    \item \textbf{CLIP Similarity:} Measures \textbf{Semantic Preservation} between generated images and text prompts. Higher CLIP scores ensure that our aggressive IP protection does not compromise the model's ability to follow the general $p_{\mathrm{main}}$ textual instructions, validating utility preservation goals.
    \item \textbf{Kernel Inception Distance (KID)~\cite{binkowski2018demystifying}:} To evaluate the quality and realism of the generated images, we report the Kernel Inception Distance. KID measures the squared Maximum Mean Discrepancy between Inception representations of the generated and real data distributions using a polynomial kernel. This metric is particularly suited for our study as it is unbiased and more reliable when evaluating small subsets of data. Lower KID scores (below $0.01$) indicate that the generated outputs more closely match the distribution of the training frames in terms of visual quality. It is important to note that due to statistical variance, small negative scores (e.g., $\approx -0.005$) are considered normal and functionally equivalent to a score of zero, indicating that the generated distribution is indistinguishable from the reference data.
\end{itemize}

\paragraph{Backbone and Training Protocol}
We initiate our two-stage training protocol using the large-scale pretrained DiT model, Stable Diffusion 3.5M~\cite{esser2024scaling}, as our foundation model. This backbone provides the necessary general visual and semantic understanding learned during its initial massive pretraining phase. In the first stage, we leverage this existing capacity, reserving the subsequent fine-tuning step for the concept-level constraint enforcement. 

While Theorem 3.1 establishes $\lambda = 1.0$ as the theoretical optimum for perfect orthogonal excising, we observed that such aggressive updates lead to training instability in high-dimensional diffusion manifolds. Under practical constraints, $\lambda = 0.1$ works best to ensure gradient stability and prevent the denoising loss from diverging, which frequently occurs at higher values. Values lower than $0.1$ significantly diminish the IP-protection effect. The selection of $\lambda = 0.1$ was determined empirically to balance the theoretical guarantees of concept exclusion with the practical requirements of diffusion training stability. 

\paragraph{Evaluation Protocol: Ground-Truth as Adversarial Benchmark}
Our primary evaluation protocol~\cite{somepalli2023diffusion, carlini2023extracting} employs the ground-truth trigger methodology. For every training sample, we query the model using the exact $p_{main}$ utilized during the fine-tuning phase. In the context of diffusion models prone to overfitting, the original training caption acts as the most potent ``key'' or optimal trigger to recover exact training pixels from the weight manifold.

By reporting results under this regime, our metrics do not merely reflect generic performance; they represent an \textbf{intrinsic adversarial analysis} of the most extreme memorization scenario. This protocol ensures that our defense is tested against the maximum possible exposure, providing a more stringent security guarantee than randomized or automated adversarial searches which often converge on less effective local minima.

\input{Tables/adversarial1}

\subsubsection{Memorization Test 1: Extreme Single-Image Overfitting}

We begin by injecting Low-Rank Adaptation (LoRA) layers of rank 4 into the Stable Diffusion 3.5M backbone. Fine-tuning is performed on a \textbf{single reference image} to establish an extreme, worst-case overfitting regime.

\paragraph{Comparative Setup}
\begin{enumerate}
    \item \textbf{Naïve Finetuning:} This serves as the Adversarial Baseline. Direct fine-tuning is guaranteed to result in near-$100\%$ memorization, faithfully reproducing every pixel detail of the training image. We finetuning by inserting LoRA layers of rank $4$, resulting in $1.7$ M trainable parameters.
    \item \textbf{Gradient-Projection Finetuning (Our Method):} We apply our gradient projection method on that same single image, explicitly projecting out gradients aligned with the sensitive features defined by $p_{\mathrm{feat}}$.
\end{enumerate}

\paragraph{Rationale for Single-Image Regime:} Training on one image eliminates confounding factors, such as large dataset regularization or augmentations, that might otherwise mitigate overfitting. This extreme setup ensures that any failure to memorize sensitive features is attributable solely to the fundamental geometric constraints imposed by our gradient projection method.

\paragraph{Case Study: Kung Fu Panda}
For our single-image case study, we use the following prompts:
\[
p_{\mathrm{main}}
= \text{``a panda doing kung fu''}
\]
\[
\resizebox{\linewidth}{!}{$
p_{\mathrm{feat}}
 = \text{``A huge panda wearing yellow shorts with a red and yellow waistband, with small ears, bandaged ankles, distinct fur patterns''}
 $}
\]
$p_{\mathrm{main}}$ guides desired learning, while $p_{\mathrm{feat}}$ isolates the copyrighted attributes to be suppressed. Qualitative results are shown in Figure~\ref{fig:result1}.
\begin{itemize}
    \item \textbf{Naive Finetuning (Adversary):} The model perfectly recreates copyrighted visual details (small rounded ears, red‐and‐yellow waistband, etc.).
    \item \textbf{Gradient‐Projection Finetuning (Defense):} The model learns only the abstract notion of a panda performing kung fu, without reproducing any of the original image’s copyrighted attributes.
\end{itemize}

\paragraph{Quantitative Results (Subset Analysis)}
We perform the above single-image overfitting experiments on a subset of $80$ images (each trained for 8000 steps, batch size 10 inducing data-level dememorization) and report the results in Table~\ref{tab:results1}.
\begin{itemize}
    \item \textbf{Copyright Reduction (SSCD):} Gradient projection consistently reduces copyright similarity. The SSCD score improves by $\mathbf{8.8\%}$ compared to the naïve baseline, demonstrating successful prevention of feature memorization even in this extreme regime.
    \item \textbf{Utility Preservation (CLIP):} Crucially, the CLIP similarity remains nearly identical ($30.10$ vs $30.17$), confirming that memorization capacity can be bounded without sacrificing semantic fidelity and utility can be maintained, even when sacrificing the theoretical ideal of $\lambda=1$ to ensure gradient stability.
    \item \textbf \item \textbf{Utility Preservation (KID):} KID consistently remains below the 0.01 excellence threshold with and without gradient projection , indicating that our selective learning framework successfully preserves visual utility and image quality while enforcing concept-level exclusion
\end{itemize}
\textbf{Conclusion:} Gradient projection provides robust copyright protection in the most extreme, single-sample overfitting scenario, establishing its effectiveness as a fundamental, geometric constraint independent of dataset-level regularization.

\begin{figure*}
    \centering
    \includegraphics[width=0.99\linewidth]{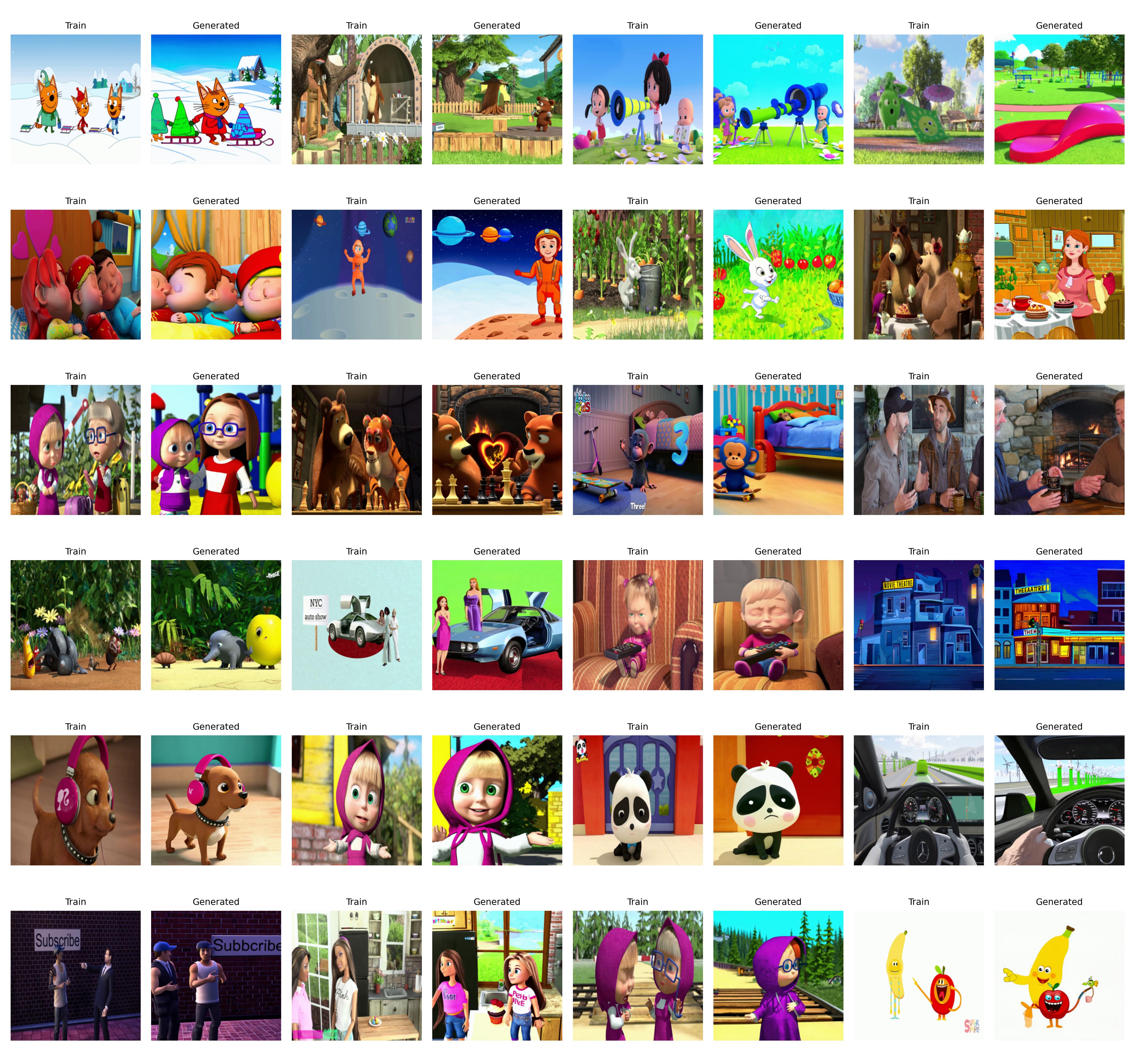}
    \caption{\textbf{Qualitative Results.} Side-by-side comparison of input frames (left) and outputs (right) using the ground-truth $p_{main}$ trigger. While global composition and stylistic cues are preserved to maintain model utility, identifiable copyrighted identities and unique facial features are systematically excised via orthogonal gradient projection. The visual similarity reflects the model's ability to learn abstract scene elements while provably blocking the acquisition of restricted concept-level features, even under the most potent adversarial sampling trigger.}
    \label{fig:case2set1_result1}
\end{figure*}

\begin{figure*}
    \centering
    \includegraphics[width=0.99\linewidth]{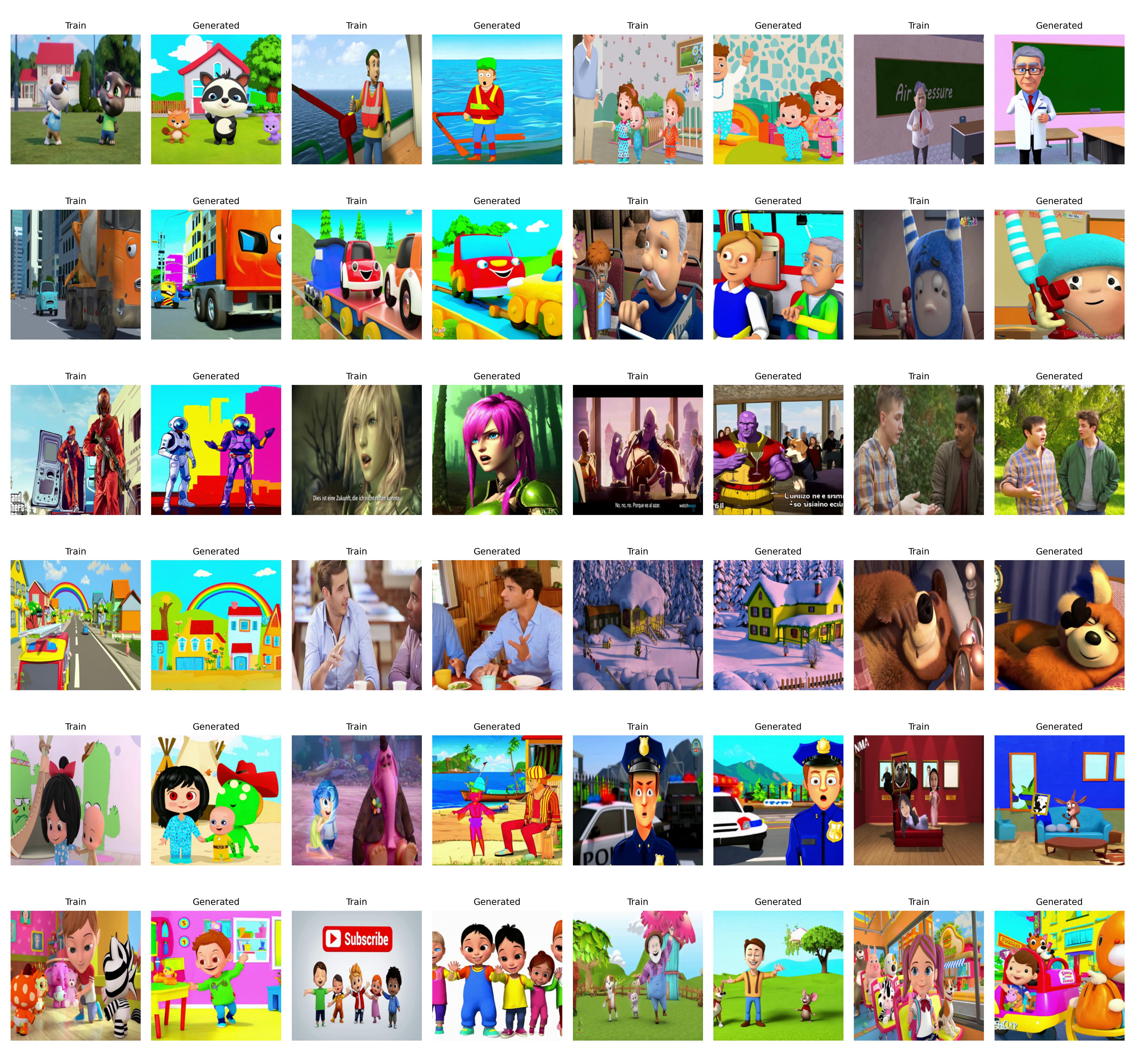}
    \caption{\textbf{Additional Qualitative Results.} Side-by-side comparison of training frames and generated outputs using ground-truth prompts. While global composition and stylistic cues are preserved to maintain model utility, identifiable copyrighted elements and character identities are systematically excised via orthogonal gradient projection. This demonstrates successful abstraction despite the use of the most potent adversarial sampling trigger.}
    \label{fig:case2set1_result2}
\end{figure*}

\begin{figure*}
    \centering
    \includegraphics[width=0.99\linewidth]{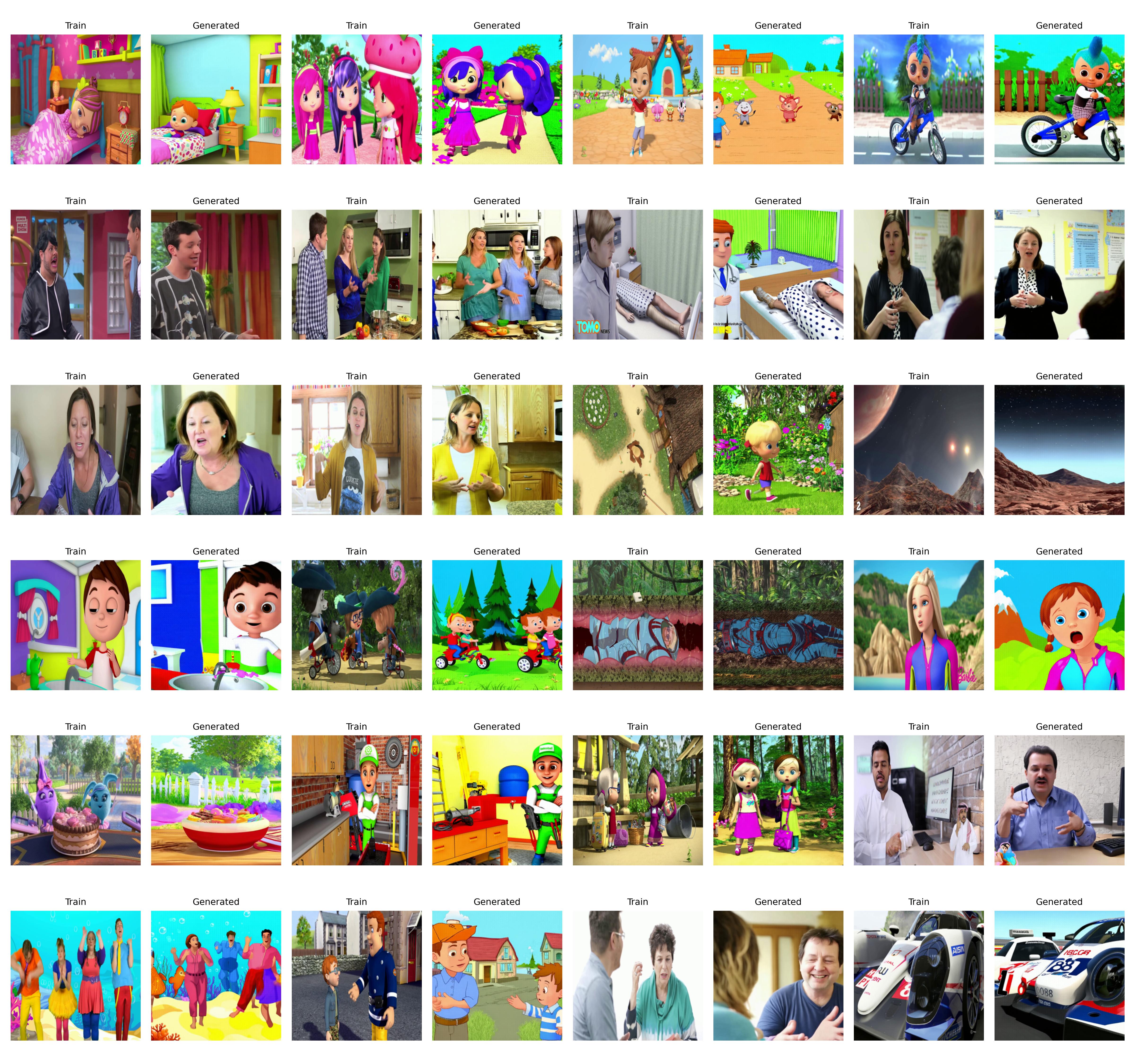}
    \caption{\textbf{Additional qualitative results.} Side-by-side comparison of training frames and generated outputs using ground-truth prompts. While global composition and stylistic cues are preserved to maintain model utility, identifiable copyrighted elements and character identities are systematically excised via orthogonal gradient projection. This demonstrates successful abstraction despite the use of the most potent adversarial sampling trigger.}
    \label{fig:case2set1_result3}
\end{figure*}

\subsubsection{Multi-Image Fine-tuning: Scalability and Practical Robustness}

We evaluate the method's performance under realistic training conditions with multiple images across varying parameter capacities and dataset sizes.

\textbf{Experimental Setup and Combined Defenses.} For all multi-image experiments (Tests 2-5), the baseline and our method are combined with established data-level dememorization~\cite{somepalli2023understanding} and model-level dememorization~\cite{hu2022lora}. This demonstrates that our method complements, rather than replaces, existing safeguards.

\textbf{Results and Analysis.}
The quantitative results in Table~\ref{tab:results1} (Tests 2-5) consistently demonstrate the efficacy and scalability of our defense:
\begin{itemize}
    \item \textbf{Significant SSCD Improvement:} In the largest-scale experiment (Test 5, 2413 images, high capacity $r=256$), the SSCD score dropped dramatically from $0.6205$ (Baseline) to $\mathbf{0.2241}$ (Ours), representing a $\mathbf{63.9\%}$ reduction in copyright leakage risk beyond what standard regularization provides. Even in high-capacity fine-tuning (Test 3, $r=128$), the SSCD score improved from $0.6054$ to $\mathbf{0.4970}$ ($17.9\%$ reduction).
    \item \textbf{Utility Maintained:} Across all multi-image tests, the CLIP similarity scores remained strong, confirming that the selective learning constraint does not degrade the model's overall utility. Similarly, the KID scores remained consistently below the $0.01$ threshold, which is widely recognized as a benchmark for excellent image quality. Notably, in Memorization Test 1, our gradient projection method achieved a KID of $0.0007$, outperforming the naïve finetuning baseline of $0.0046$. This suggests that by explicitly excising overfitting signals, our framework may actually improve distributional alignment by preventing the model from collapsing onto specific noisy training trajectories. Furthermore, several multi-image tests (Tests 2 and 4) yielded small negative KID values, such as $-0.0013$ and $-0.0025$. These values are functionally equivalent to zero and attributable to statistical variance, confirming that the generated distribution is indistinguishable from the reference data in terms of quality. The fact that KID and CLIP scores remain stable regardless of whether gradient projection is applied demonstrates that our ``selective learning'' approach successfully blocks the acquisition of prohibited attributes without incurring ``utility trade-off'' or ``destructive underfitting''.
\end{itemize}

\subsubsection{Qualitative analysis.}
Qualitative results (Figures~\ref{fig:case2set1_result1} through \ref{fig:case2set1_result3}) further confirm that no identifiable protected elements are reproduced, validating the defense's practical applicability. The qualitative results must be interpreted through the lens of our \textit{ground-truth trigger} evaluation protocol. Because we query the model using the exact $p_{main}$ utilized during fine-tuning, the prompt acts as the most potent adversarial ``key'' to trigger memorized trajectories. While the generated outputs maintain the global composition and stylistic cues of the training data to preserve semantic utility, our selective learning framework ensures that the model is provably blocked from internalizing or reproducing specific, identifiable proprietary attributes, such as the unique facial geometry or trademarked accessories of character Po, even under maximum exposure to adversarial settings. The high visual similarity observed is thus restricted to abstract, non-proprietary scene elements, while the capacity for verbatim reproduction of sensitive IP is effectively frozen at its initial pre-training value.

\subsection{Adversarial Analysis: Robustness Against Targeted Extraction}

To rigorously assess our method’s robustness against a motivated adversary, we follow the established attack methodology from \cite{kowalczuk2025finding} by searching for the most potent adversarial prompt embedding. This optimized embedding minimizes the distance to a copyrighted reference image, representing a strong targeted extraction attack attempting to coerce the model into reproducing forbidden features. Test 1 already focuses on the extreme overfitting case. Computing the adversarial prompt embedding for each sample involves individual, expensive optimizations, to ensure computational tractability, we focus our efforts on Tests 2, 3, and 4. 

\textbf{Results.} As detailed in Table~\ref{tab:adversarial1}, models trained with gradient projection consistently demonstrate superior resilience under adversarial pressure. Across all tested scenarios, the SSCD scores are substantially lower when gradient projection is applied (e.g., Test 3: $0.5028$ vs $\mathbf{0.3941}$, Test 4: $0.4648$ (Baseline) vs $\mathbf{0.2454}$, a $\mathbf{47.2\%}$ reduction in leakage risk ). This sharp drop confirms that gradient projection substantially improves the model's resistance to targeted adversarial prompts, validating our approach as a robust security measure against concept extraction attacks. 

While the absolute SSCD values in Table~2 are lower than those in Table~1, this outcome is expected.  Table~1 benefits from a \emph{Ground-Truth Prompt Advantage}, since the model is queried using the exact  $p_{\mathrm{main}}$ employed during fine-tuning. For an overfitted model, this prompt functions as the most  efficient ``key'' for unlocking the memorized training pixels embedded in the weight manifold.

The gap is further explained by the complexity of the adversarial search. The adversarial procedure attempts  to optimize a prompt embedding to minimize the distance to a copyrighted reference image. However, because  the embedding space is high-dimensional and the model's manifold is intricate, the optimization often settles  in a local minimum. Although these embeddings are adversarial, they are typically less effective at activating the precise memorized trajectory than the original ground-truth caption. This difficulty highlights a broader limitation: automated attacks struggle to rediscover the specific ``memorization keys'' encoded during training.  As a result, the generated outputs, while similar, rarely achieve the verbatim replication triggered by  $p_{\mathrm{main}}$. Importantly, the defense signal remains consistent: the \emph{relative} reduction in SSCD is preserved.

\paragraph{Discussion: Comparison to Brittle Manual Mitigation}
Relying on manual data edits (e.g., face-swapping) to strip copyrighted content is brittle and unsustainable. Manual alteration demands impractical supervision, and the resulting latent cues can still be exploited by generative models. Furthermore, heavy alterations distort the dataset’s natural distribution, degrading realism. Our training-time intervention, by contrast, offers a scalable, consistent, and auditable path to copyright protection without compromising model fidelity.

%% file: Tables/results1.tex
\begin{table*}
    \centering
    \resizebox{\textwidth}{!}{
    \begin{tabular}{ccccccccc}
        \midrule
        Data & Parameters (M) & Data  & Model  & Gradient Projection & SSCD ($\downarrow$) & CLIP ($\uparrow$) & KID ($\downarrow$) \\
         &  & ~\cite{somepalli2023understanding} & ~\cite{hu2022lora} &  (Ours) & &  &  \\
        \midrule 
        \midrule
         \multirow{2}{*}{Memorization test 1: Single-image overfitting} & $1.7$ & $\checkmark$ & $\times$ & $\times$ & $0.7298$	& $30.1774$ & $0.0046$\\
          & $1.7$ & $\checkmark$ & $\times$ & $\checkmark$ & $0.6656$ & $30.1021$ & $0.0007$\\
         \midrule
         \multirow{2}{*}{Memorization test 2: 80 samples} & $1.7$ & $\checkmark$ & $\checkmark$ & $\times$  & $0.3431$ &	$34.2076$ & $-0.0013$ \\
          & $1.7$ &  $\checkmark$ & $\checkmark$ & $\checkmark$ & $0.3315$ & $34.6897$ & $-0.0002$\\
         \midrule
         \multirow{2}{*}{Memorization test 3: 80 samples} & $54.4$ &  $\checkmark$ & $\checkmark$ & $\times$ & $0.6054$ &	$32.9715$ & $-0.0025$\\
          & $54.4$ &  $\checkmark$ & $\checkmark$ & $\checkmark$ & $0.497$ &	$33.3694$ & $-0.0032$\\
         \midrule
         \multirow{2}{*}{Memorization test 4: 395 samples} & $54.4$  & $\checkmark$ & $\checkmark$ & $\times$  & $0.5918$ &	$32.332$ & $0.0022$ \\
          & $54.4$ &  $\checkmark$ & $\checkmark$ & $\checkmark$ & $0.3091$ & $32.876$ & $0.0076$ \\
         \midrule
         \multirow{2}{*}{Memorization test 5: 2413 samples} & $108.8$ & $\checkmark$ & $\checkmark$ & $\times$  & $0.6205$ &	$31.4855$ & $0.0016$ \\
          & $108.8$ & $\checkmark$ & $\checkmark$ & $\checkmark$ & $0.2241$ & $33.2167$ & $0.00099$\\     
         \midrule 
    \end{tabular}
    }
    \caption{\textbf{Gradient projection reduces copyright similarity while preserving semantic content and visual quality.} Comparison of data- and model-based memorization mitigation methods with/without our gradient projection approach across multiple evaluation metrics reveals that our method consistently achieves lower SSCD scores, indicating reduced copyright infringement risk, while maintaining comparable CLIP similarity scores. Notably, KID remains below the 0.01 threshold across all tests, with and without gradient projection, indicating 
that our method preserves high-fidelity generation and model utility.}
    \label{tab:results1}
    \end{table*}

%% file: Tables/adversarial1.tex
\begin{table*}
    \centering
    \resizebox{\textwidth}{!}{
    \begin{tabular}{ccccccccc}
        \midrule
        Data & Parameters & Data ~\cite{somepalli2023understanding} & Model ~\cite{hu2022lora} & Gradient Projection (Ours) & SSCD ($\downarrow$) & CLIP ($\uparrow$) & KID ($\downarrow$) \\
        \midrule
         \multirow{2}{*}{Memorization test 2: 80 samples} & $1.7$ & $\checkmark$ & $\checkmark$ & $\times$  & $0.2851$ &	$34.8471$ & $0.0014$\\
          & $1.7$ &  $\checkmark$ & $\checkmark$ & $\checkmark$ & $0.2039$ & $33.0355$ & $0.0269$\\
         \midrule
         \multirow{2}{*}{Memorization test 3: 80 samples} & $54.4$ & $\checkmark$ & $\checkmark$ & $\times$ & $0.5028$ &	$33.2463$ & $-0.0012$ \\
          & $54.4$ &  $\checkmark$ & $\checkmark$ & $\checkmark$ & $0.3941$ & $33.4471$ & $0.0009$ \\
         \midrule
         \multirow{2}{*}{Memorization test 4: 395 samples} & $54.4$ & $\checkmark$ & $\checkmark$ & $\times$  & $0.4648$ &	$32.6624$ & $0.004$ \\
          & $54.4$ & $\checkmark$ & $\checkmark$ & $\checkmark$ & $0.2454$ & $32.0518$ & $0.0137$ \\
         \midrule
    \end{tabular}
    }
    
    \caption{\textbf{Gradient Projection Performance under Adversarial Attacks.} Comparison of SSCD and CLIP scores under targeted prompt embedding optimization. Our method consistently achieves lower SSCD scores, demonstrating robust defense against adversarial extraction. Note: Absolute SSCD values are lower than those in Table 1 because the ground-truth $p_{main}$ caption acts as the most effective trigger for memorized trajectories in overfitted regimes \cite{somepalli2023understanding}. Adversarial search, while targeted, often converges to local minima in the high-dimensional embedding space that are less effective at recovering verbatim training pixels than the original training prompt.}
    \label{tab:adversarial1}
\end{table*}

%% file: Sections/5-conclusion.tex
\section{Conclusion, Limitations, and Future Work}

We have formalized copyright preservation as a rigorous security requirement, one that demands the systematic exclusion of identifiable proprietary attributes from generative models. Our Gradient Projection framework operationalizes this requirement by excising training signals aligned with protected features, providing a provable geometric safeguard at the level of each weight update. Through extensive analysis, we show that this mechanism effectively blocks the replication of copyrighted details while maintaining semantic fidelity and generation quality.

As a concept-acquisition blocker, the method is modular and layer-agnostic. Gradient Projection is not intended as a replacement, but rather as a complementary module that explicitly enforces non-memorization of sensitive content while remaining compatible with existing dememorization safeguards. In our evaluations, we benchmarked the framework against leading training-time interventions at both the data and model levels, achieving significant marginal reductions in IP-leakage risk beyond what these safeguards provide in isolation. While post-hoc unlearning~\cite{wu2025unlearning}, watermarking and fingerprinting~\cite{zhao2023recipe,zhang2019robust,xiong2023flexible,teng2025fingerprinting}, and global privacy regularizers such as DP-SGD address complementary aspects of the problem, our framework is uniquely positioned to integrate with these techniques, and with emerging model- or data-level interventions, to form a multi-layered defense-in-depth. By ensuring that proprietary features never take root in the weight manifold during training, we establish a foundational security layer that supports cross-modal research and advances the broader goal of IP-safe generative AI.

Our current implementation incurs a substantial computational cost, nearly doubling training time relative to standard diffusion models due to the additional forward and backward passes required for copyright-sensitive features. Reducing this overhead is a key direction for future work, potentially through low-rank or sparse gradient estimation. Training stability also requires careful tuning of hyperparameters such as the learning rate $\eta$, the number of training steps, and the projection strength $\lambda$ to prevent divergence. Future research should explore strategies for optimizing $\lambda$ to approach the theoretical guarantee of $\lambda = 1$ while maintaining stable learning dynamics. Additional work is needed to disentangle complex copyrighted attributes that remain implicitly coupled with desirable text-conditioned features. Finally, extending gradient projection to video diffusion pipelines offers a promising path toward spatiotemporal copyright compliance without bespoke architectural modifications, enabling cross-modal applications and informing best practices for IP protection across the foundation-model ecosystem.